\documentclass[runningheads]{llncs}

\usepackage{eccv}

\usepackage{eccvabbrv}

\usepackage{graphicx}
\usepackage{booktabs}

\usepackage[accsupp]{axessibility}  

\usepackage[pagebackref,breaklinks,colorlinks,citecolor=eccvblue]{hyperref}

\usepackage{orcidlink}

\usepackage{float}
\usepackage{multirow}
\usepackage{bbm}
\usepackage{algorithm}
\usepackage{algpseudocode}
\usepackage[normalem]{ulem}
\DeclareMathOperator{\E}{\mathbb{E}}

\DeclareMathOperator{\Var}{Var}

\newcommand{\NIS}{\text{NIS}}
\newcommand{\pckm}{\text{pckm}}
\newcommand{\km}{\text{km}}
\newcommand{\CC}{\text{CC}}
\DeclareMathOperator{\D}{\mathcal{D}}
\DeclareMathOperator{\Y}{\mathcal{Y}}
\DeclareMathOperator{\C}{\mathcal{C}}

\DeclareMathOperator{\M}{\mathcal{M}}

\usepackage{graphbox}

\UseRawInputEncoding   
\begin{document}

\title{Human-in-the-Loop Visual Re-ID for Population Size Estimation}

\titlerunning{Human-in-the-Loop Visual Re-ID for Population Size Estimation}

\author{Gustavo Perez\inst{1,2}\orcidlink{0000-0003-3880-8075} \and
Daniel Sheldon\inst{2} \and
Grant Van Horn\inst{2} \and 
Subhransu Maji\inst{2}\orcidlink{0000-0002-3869-9334}}

\authorrunning{G.~Perez et al.}

\institute{University of California, Berkeley \\
\email{gperezs@berkeley.edu}\\
\and
University of Massachusetts, Amherst\\
\email{\{sheldon,gvh,smaji\}@cs.umass.edu}}

\maketitle

\begin{abstract}

Computer vision-based re-identification (Re-ID) systems are increasingly being deployed for estimating population size in large image collections. However, the estimated size can be significantly inaccurate when the task is challenging or when deployed on data from new distributions. We propose a human-in-the-loop approach for estimating population size driven by a pairwise similarity derived from an off-the-shelf Re-ID system. Our approach, based on nested importance sampling, selects pairs of images for human vetting driven by the pairwise similarity, and produces asymptotically unbiased population size estimates with associated confidence intervals. We perform experiments on various animal Re-ID datasets and demonstrate that our method outperforms strong baselines and active clustering approaches. In many cases, we are able to reduce the error rates of the estimated size from around 80\% using CV alone to less than 20\% by vetting a fraction (often less than 0.002\%) of the total pairs. The cost of vetting reduces with the increase in accuracy and provides a practical approach for population size estimation within a desired tolerance when deploying Re-ID systems.\footnote{Code available at: \href{https://github.com/cvl-umass/counting-clusters}{https://github.com/cvl-umass/counting-clusters}}

\end{abstract}

\section{Introduction}
\label{sec:intro}
\vspace{-5pt}

\begin{figure}[t]
\centering
\begin{tabular}{cc}
\multirow{-5.5}{*}[1pt]{\includegraphics[width=0.53\linewidth]{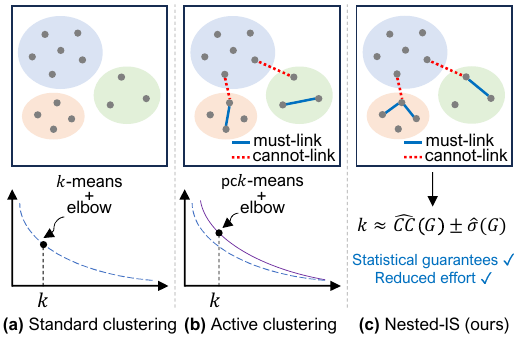}} & 
\hspace{0pt} \includegraphics[width=0.44\linewidth]{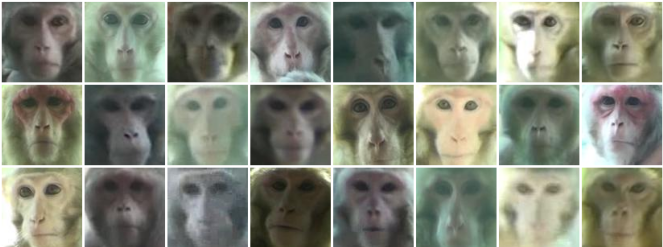}\\
& \hspace{0pt} \includegraphics[width=0.45\linewidth]{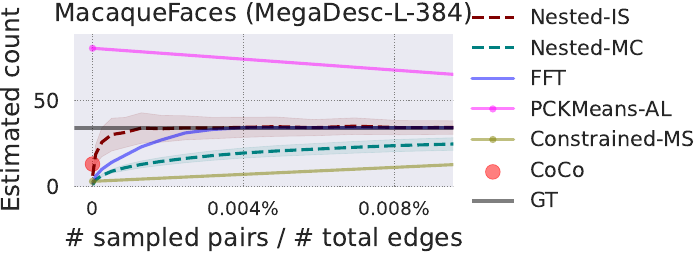}

\end{tabular}

\vspace{-10pt}
\caption{
\textbf{Estimating Population Size Using a Re-ID System}. \textbf{(a)} A simple approach involves using $k$-means clustering on image embeddings derived from the Re-ID system and selecting the optimal $k$ using the ``elbow heuristic." \textbf{(b)} Active clustering (e.g., pck-means~\cite{basu2004}) employs pairwise constraints to enhance clustering accuracy. \textbf{(c)} Our method leverages nested importance sampling to produce asymptotically unbiased estimates and confidence intervals on $k$ directly. \textbf{(Right)} On the MacaqueFaces dataset~\cite{MacaqueFaces}, our approach (Nested-IS) converges to the true $k=34$ with fewer constraints than alternative methods, but also provides confidence intervals for the estimate (shown as the shaded red region) for any amount of human feedback.}
\vspace{-15pt}
\label{fig:confidenceintervals}
\end{figure}

Computer vision-based re-identification (Re-ID) is increasingly being deployed to analyze large image collections for species-level population monitoring~\cite{schneider2019past,vidal2021perspectives}. The association of individuals across camera sensors provides valuable data about animal movement, population dynamics, geographic distribution, and habitat use, which in turn can inform conservation goals and ecological models~\cite{tuia2022perspectives,steenweg2017scaling,burton2015wildlife}. Computer vision can reduce the manual labor associated with individual identification, and there is significant effort from the community to develop tools that facilitate animal Re-ID across a wide range of animal species~\cite{beery2019efficient,crall2013hotspotter,bolger2012computer,anderson2010computer,megadesc,weideman2020extracting}.

We consider the task of estimating population size, that is, the number of individuals, in a collection of images using a Re-ID system. Beyond surveys, this quantity can inform techniques for category discovery and incremental learning of AI systems~\cite{gcd,pimgcd,joseph2022}. Since the number of individuals is often unbounded, Re-ID is cast as a pairwise classification task and involves predicting whether two images are of the same individual or not. However, estimating the population size from pairwise similarities is non-trivial. One approach is to use a clustering algorithm such as $k$-means with the ``elbow heuristic''~\cite{kmeans_elbow} to pick $k$. However, there are many algorithms and heuristics to choose from, making the process subjective. Estimated pairwise similarity can also be a poor approximation to the true similarity, especially when models are deployed out-of-distribution or on challenging tasks where the performance of Re-ID is low. 
For instance, using $k$-means with the elbow heuristic on the MacaqueFaces dataset~\cite{MacaqueFaces} using the state-of-the-art ``MegaDescriptor''~\cite{megadesc} results in 80 clusters when the true number of individuals is 34, as seen in Fig.~\ref{fig:confidenceintervals}. Table~\ref{table:datasets} shows that estimated population size across a variety of animal species and clustering approaches are sometimes off by factor of two or more, which can impact downstream tasks.

Our approach employs statistical estimation to compute the population size using human feedback on pairwise similarity. We develop an estimator based on nested importance sampling, each iteration driven by the approximate pairwise similarity. Feedback in the form of ``same'' or ``not'' on a set of sampled edges is used to estimate the number of clusters as well as to provide a confidence interval. We theoretically demonstrate that the estimator is unbiased, i.e., it converges to the true number of clusters in expectation, and an end user can stop screening when the confidence intervals are sufficiently small. Additionally, we contribute a strong baseline based on farthest-first traversal (FFT) to directly estimate the population size, which, although it does not offer statistical guarantees, outperforms alternative active clustering approaches. Our methods can employ any off-the-shelf image similarity and do not involve fine-tuning deep representations, thus can be practical for non-AI experts.

We conduct experiments on seven datasets where the goal is to estimate the number of individuals they contain. These datasets span animal species and exhibit different data distributions, including both short and long-tailed, with varying number of individuals. We also experiment with different image representations from the MegaDescriptor library~\cite{megadesc}, which provides a unified deep network for various animal Re-ID tasks. The accuracy of the Re-ID system also varies across datasets and are assumed to be unknown, providing a realistic use case for deploying the models. 
Our approach is more accurate for the same amount of human supervision (sampled pairs) than competing baselines. We also show that the estimator has low bias and produces accurate confidence intervals. In summary, our main contributions are:  
\begin{itemize}
    \item We study the challenges involved in using Re-ID systems for estimating population size across a range of animal Re-ID datasets.
    \item We propose a novel approach combining nested importance sampling with human-in-the-loop feedback. 
    that produces asymptotically unbiased estimates of population size (\S~\ref{sec:method}).
    \item We design confidence intervals that provide intuitive feedback for guiding human effort and setting stopping conditions %
    (\S~\ref{sec:nestedis}).
    \item We report extensive experiments on a benchmark of seven animal Re-ID datasets with different data distributions, demonstrating that our method achieves significantly lower error than strong baselines with similar human effort. In most cases our approach produces estimates close to the true population size using human feedback on less than 0.004\% of all pairs (\S~\ref{sec:evaluation}).
    \item We use our framework to estimate the number of categories in a dataset for generalized category discovery~\cite{gcd,pimgcd,joseph2022} and measure the impact on clustering accuracy,  
    on animal Re-ID and fine-grained classification 
    (\S~\ref{sec:evaluation}).
\end{itemize}

\vspace{-5pt}
\section{Related Work}\label{sec:relatedwork}
\vspace{-5pt}
The task of estimating population size in a dataset is closely related to the problem of grouping individuals using a Re-ID system. However, it is possible to estimate the population size using statistical approaches, such as ours, without fully solving the grouping problem. There is a vast literature in computer vision, especially in face recognition, on building accurate Re-ID systems, as well as using these representations to group individuals. We briefly review these approaches. Our approach focuses on animal Re-ID tasks. This task can be relatively easy for some species that have distinctive patterns but significantly more challenging for others. We review prior work on animal Re-ID and introduce the benchmarks and representations we experiment with.

\vspace{-5pt}
\paragraph{Re-ID Systems} There is significant prior work on person Re-ID based on face and whole-body recognition (see~\cite{erik_facesurvey} for a more complete survey). Early techniques, such as Eigenfaces~\cite{turk1991eigenfaces}, have been replaced by modern deep learning approaches~\cite{parkhi2015deep,liu2017sphereface,liu2017sphereface} trained on large datasets~\cite{guo2016ms,huang2008labeled}. While significant advances have been made, there are also known issues with poor generalization and bias of Re-ID systems in the presence of demographic shifts~\cite{learned2020facial}. In comparison, techniques for animal Re-ID are less mature. These techniques range from SIFT~\cite{sift} and Superpoint-based~\cite{superpoint} matching approaches (e.g., WildID~\cite{bolger2012computer} and HotSpotter~\cite{crall2013hotspotter}) that work relatively well for species with characteristic patterns on their skin or coats (e.g., Zebras and Jaguars), to complex pipelines designed to identify visual characteristics specific to particular species (e.g., whisker patterns of polar bears~\cite{anderson2010computer}). Deep learning approaches have also been developed for some species (e.g., Chimpanzees and Bears) but the training datasets are relatively small. Recent work~\cite{megadesc} has attempted to train deep learning models across species by consolidating animal Re-ID datasets and has shown that the model generalizes across species and significantly outperforms prior work, including off-the-shelf image representations such as CLIP~\cite{clip} and DINOv2~\cite{dino}. We utilize a set of datasets from their collection (Table~\ref{table:datasets}) and base our population size estimation on the various pre-trained deep networks they provide. Their best-performing model is a Swin-transformer~\cite{SwinT} trained with ArcFace loss~\cite{arcface} on the collection of datasets. However, the problem is far from solved -- for example, on the WhaleSharkID dataset~\cite{WhaleSharkID}, the performance of the Re-ID system is around 62\%, which poses a challenge for population size estimation.

\vspace{-5pt}
\paragraph{Clustering}
The population size can be estimated by determining the number of clusters in a dataset. A common approach involves using \textit{clustering algorithms}, such as $k$-means~\cite{kmeans}, mean-shift~\cite{meanshift}, which operate directly on embeddings, or graph-based approaches~\cite{graphcuts, normalized-cuts, erdos-renyi-clustering} that incorporate pairwise similarity, and determining the number of clusters based on a \textit{heuristic}. For example, one can compute the within-cluster sum of squares (WCSS) for different values of $k$ in $k$-means and select the optimal one based on the ``elbow heuristic'' -- the point in the curve where improvements diminish (see Fig.~\ref{fig:confidenceintervals}). However, there are many clustering algorithms and heuristics available, making the process subjective. We find that population size estimates using these methods can be significantly inaccurate, even with state-of-the-art embeddings (see Table~\ref{table:datasets}). More complex approaches, particularly those incorporating tracking information and sophisticated graph-based clustering (e.g.,~\cite{erdos-renyi-clustering}), have been proposed for grouping in videos. Our primary focus is on image-based approaches using simple clustering algorithms as they are broadly applicable.

\vspace{-5pt}
\paragraph{Active Clustering}
Human feedback can be incorporated to improve clustering in various ways. One can fine-tune the deep network using metric learning~\cite{kulis2013metriclearning,kaya2019deepmetric} approaches to improve the underlying similarity using human feedback, for example thorough pairwise~\cite{metric_siamese} and triplet-based~\cite{metric_triplet} learning. However, these methods require significant compute resources and expertise to set various hyperparameters associated with training, and may not be practical for non-experts. Moreover, fine-tuning on uncurated data often encountered in a real-world deployment of the Re-ID system is rarely effective. Hence, we focus on active clustering approaches that incorporate constraints to improve clustering for a fixed embedding. We compare with a constrained $k$-means algorithms that use `must-link' and `cannot-link' constraints within the $k$-means~\cite{basu2004}. Constraints can be selected using a \emph{farthest-first} scheme~\cite{basu2004}, or by running $k$-means with a large $k$ and picking constraints to merge the small clusters into larger ones~\cite{cobra}. We also develop a baseline based on the farthest-first traversal (FFT) scheme that is directly aimed at estimating the number of clusters. 
While these approaches improve over the baseline $k$-means algorithm they do not provide statistical guarantees on the estimate. Our proposed sampling-based approach results in a consistent estimator of the cluster count and produces accurate confidence intervals (\S~\ref{sec:results}). The amount of human effort required depends both on the quality of the pairwise similarity as well as the level of precision needed for estimation.

\vspace{-5pt}
\paragraph{Related Problems} Apart from population surveys, estimating the number of clusters is often the primary goal for many tasks. Life-long learning systems must discover and learn novel categories during long-term deployment. In generalized category discovery (GCD)~\cite{gcd,pimgcd}, the goal is to cluster images given the labels for a subset of images in the presence of novel categories. This problem is more challenging than traditional semi-supervised learning due to the open-world setting. However, we find that the performance of existing GCD approaches is limited by the accuracy of estimating the number categories in a dataset. Often, existing approaches make the unrealistic assumption that this quantity is known. We show that our approach for estimating the number of clusters has a higher impact on improving GCD than using active clustering algorithms. 
Our work is also related to approaches that improve statistical estimation by combining human effort and model predictions, such as ISCount~\cite{iscount}, DISCount~\cite{discount}, and Prediction-Powered Inference~\cite{ppi}. For example, both ISCount and DISCount use importance sampling to estimate the population mean, and our work extends this idea to a pairwise setting.

\begin{figure}[t]
    \centering
    \includegraphics[width=0.5\linewidth]{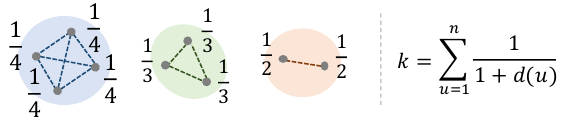}\\
    \vspace{-5pt}
    \caption{\textbf{Counting clusters in a graph.} The number of clusters $k = \sum_{u=1}^n 1/(1+d(u))$, where $d(u)$ is degree of node $u$. In this example $k = 4 \times 1/4 + 3 \times 1/3 + 2 \times 1/2 = 3$.}
    \vspace{-10pt}
    \label{fig:degrees}
\end{figure}

\vspace{-5pt}
\section{Method}\label{sec:method}
\vspace{-5pt}
Assume that we have a collection of images $\D=\{x_i\}_{i=1}^n$ where each image $x_i$ belongs to a cluster $y_i \in {\cal Y}$. Our goal is to calculate the total number of clusters $K=|{\cal Y}|$ in $\D$. We assume labels $y_i \in \Y$ are unknown but we have access to an image embedding $\Phi(\cdot)$ which can be used to compute an approximate pairwise similarity $\hat{s}(x,z) = f\left(\Phi(x), \Phi(z)\right)$ between images $x$ and $z$ for some function $f$ (such as the dot product). The true similarity is $s(x,z) = 1$ if images $x$ and $z$ belong to the same cluster and is $s(x,z) = 0$ otherwise. We assume this can be obtained by using human feedback in the form of ``same'' or ``not" for the pair of images. Our goal is to estimate $K$  
as accurately as possible using a small amount of human feedback given the approximate similarity $\hat{s}(x,z)$. 

\vspace{-5pt}
\begin{lemma}
\label{lemma:lemma1}
Consider a graph  $G=(V,E)$ with vertices $u,v \in V=\{1,...,n\}$ corresponding to images $x_u$, with an edge  $e_{uv} \in E$ between  $u$  and $v$  if the images \( x_u \) and \( x_v \) belong to the same cluster. Let $d(u) = |\{ e_{uv} \in E \}|$ be the degree of vertex \( u \). Then the number of clusters \( K \) in the dataset \(\mathcal{D}\) is equal to the number of connected components in \( G \), and can be computed as: 
\begin{equation}\label{eq:cc}
   K = \CC(G) = \sum_{u=1}^n \frac{1}{1+d(u)}.
\end{equation}
\end{lemma}

\vspace{-9pt}
\begin{proof}[Proof of Lemma~\ref{lemma:lemma1}]
    It is easy to see that the $G$ is a collection of cliques, with the clique containing $u$ of size $1+d(u)$. The total contribution of the vertices in this clique is $(1+d(u))\times(1/(1+d(u))) = 1$, i.e., each clique contributes a total of 1 to Eq.~\ref{eq:cc}, adding up to the total number of cliques or connected components in $G$. See Fig.~\ref{fig:degrees} for an example.
\end{proof}


\begin{figure*}[t]
    \centering
    \includegraphics[width=1.0\linewidth]{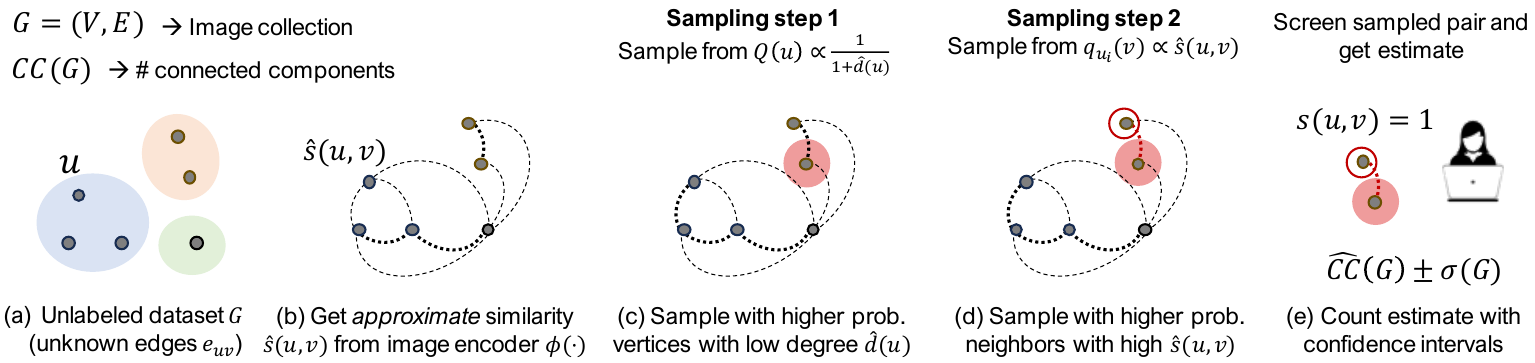}\\
    \vspace{-5pt}
    \caption{\textbf{Proposed Framework for Counting Clusters in a Dataset.} \textbf{(a)} We re-present dataset as a graph $G$ and estimate the number of connected components for an (unknown) pairwise similarity. \textbf{(b)} First, we compute an \emph{approximate} similarity between images using an embedding. \textbf{(c)} We sample vertices $u_i$ from the distribution $Q(u)$ which biases the samples towards vertices with low (estimated) degrees. \textbf{(d)} We then sample nodes of $v_{i,j}$ from $q_{u_i}(v)$ biased towards neighbors. 
    \textbf{(e)} Human feedback on the sampled pairs is used to estimate the number of clusters with confidence intervals. 
    }
    \vspace{-9pt}
    \label{fig:main}
\end{figure*}

\vspace{-10pt}
\subsection{Estimation via Nested Sampling}\label{sec:nestedsampling}
The above lemma provides a way to estimate the number of connected components $\CC(G)$ by sampling.
\vspace{-5pt}
\begin{equation}
    \CC(G) = \sum_{u=1}^n \frac{1}{1+d(u)} = n\mathbb{E}_{u\sim \text{Unif}(1,n)} \left[ \frac{1}{1+d(u)} \right].
\end{equation}
Thus, one approach to estimate the number of connected components is to sample $N$ vertices and estimate their degrees by querying a human if an edge exists between the vertex and all other $n-1$ vertices in the graph. This would require $N \times (n-1)$ queries, saving over 
$n \times (n-1)/2$ 
queries for exact estimation. However, one can estimate the degree of a vertex by sampling as well. This suggests a two-step Monte Carlo (MC) approach for estimation. First, we sample $N$ vertices $u_1, \ldots, u_N$. For each sampled vertex $u_i$, we then sample $M$ nodes $v_{i,1}, \ldots, v_{i,M}$ uniformly at random, query a human to obtain $s(u_i,v_{i,j})$ for each potential neighbor, and estimate the degree of the vertex $u_i$ as:
\vspace{-5pt}
\begin{equation} \label{eq:degree}
\hat{d}(u_i) = \frac{n-1}{M} \sum_{j =1}^M s(u_i,v_{i,j}).
\end{equation}
\vspace{-5pt}
From this, we can estimate the number of connected components as:
\vspace{-3pt}
\begin{equation}\label{eq:mc}    
    \widehat{\CC}_\text{MC}(G) = \frac{n}{N} \sum_{i=1}^N \frac{1}{1+\hat{d}(u_i)}.
\end{equation}
\vspace{-5pt}

This nested Monte Carlo estimate is asymptotically unbiased for the expected number of connected components, i.e., it converges to the true mean under mild assumptions.  
One can also construct a sample estimate of the variance and confidence intervals around the mean (we will derive the expressions for a general sampling distribution in the next section). The number of queries required to construct the estimate scales as $N \times M$, which is more efficient than the earlier approach of $N \times n$. However, the variance of the estimator can be high for graphs where the degrees of the vertices vary significantly.

\vspace{-5pt}

\begin{remark}
    Our problem is related to work on estimating the number of connected components in a graph in sub-linear time. The randomized algorithm proposed in~\cite{chazelle2005approximating,berenbrink2014estimating} use a similar sampling argument but their cost model is different. They assume the edges in the graph are provided as an adjacency list and run breadth-first-search starting from each node for a number of steps till the size of the connected components exceeds a threshold---a threshold of $1/\epsilon$ gives an $n\epsilon$ additive approximation to CC.
    In our setting the true edges are unknown---we instead have a noisy pairwise similarity---and have to pay a cost to reveal an edge, and hence the same approach is not applicable.
\end{remark}

\subsection{Estimation via Nested Importance Sampling}\label{sec:nestedis}

\vspace{-3pt}
Importance sampling~\cite{mcbook} uses a proposal distribution $q$ and replaces the expectation of a quantity $f(x)$ under $p$ as: 
\vspace{-5pt}

\begin{equation}
\E_p \left[f(x)\right] =\E_q \left[\frac{p(x)}{q(x)}f(x)\right].
\end{equation}
The equality holds for any proposal distribution $q$ that satisfies $p(x)f(x) > 0 \implies q(x) > 0$. Moreover, one can show that the optimal proposal distribution $q(x) \propto p(x) f(x)$ for which the estimator has zero variance.

\setlength{\textfloatsep}{10pt}
\begin{algorithm}[t]
\caption{Estimating cluster count $k$ using NIS}\label{alg:cap}
\begin{algorithmic}[1]
\State $N \gets$ Number of sampled vertices
\State $M \gets$ Number of sampled edges per vertex
\State $\hat{s}(u,v) \gets$ Approximate pairwise similarity
\For{i = 1,...,N}
    \State Sample $u_i \sim Q(u)$
    \State Sample $v_{i,1},...,v_{i,M} \sim q_{u_i}(v)$
    \State $\hat{d}(u_i) \gets \frac{1}{M} \sum_{j=1}^M \frac{s(u_i,v_{i,j})}{q_{u_i}(v_{i,j})}$ \ \ \ {\scriptsize// Human feedback}\Comment{Eq. (\ref{eq:Nbar})}
\EndFor
\State $\widehat{\CC}_\NIS \gets \frac{1}{N} \sum_{i=1}^N \left( \frac{1}{Q(u_i)} \times \frac{1}{1+\hat{d}(u_i)} \right)$  \Comment{Eq. (\ref{eq:Fnis})}
\end{algorithmic}
\label{alg:nis}
\end{algorithm}

We are interested in computing various expectations under a uniform distribution $p$, and one can show that the optimal proposal distribution to sample nodes for estimating $\CC(G)$ in Eq.~\ref{eq:cc} is $Q(u) \propto 1/(1+d(u))$, while that to sample edges for estimating the degree $\hat{d}(u)$ in Eq.~\ref{eq:degree} is $q_u(v) \propto s(u,v)$.

However, sampling from these distributions requires knowledge of the true $s(u,v)$ which is unknown. But we can construct proposals using the approximate pairwise similarity $\hat{s}(u,v) \in [0,1]$. This can be computed using 
the similarity between a pair of images estimated from their feature embeddings. 
Then, we 
can set $Q(u) \propto 1/(1+\sum_{v \neq u} \hat{s}(u,v))$ as the proposal distribution to sample a set of vertices $I$ of size $N$. For each sampled vertex $u_i$, we sample a set of $M$ nodes $v_{i,j}$ from the distribution $q_u(v) \propto \hat{s}(u,v)$. The $\widehat{\CC}_\NIS(G)$ obtained using importance sampling is:
\vspace{-5pt}
\begin{equation}\label{eq:Fnis}
    \widehat{\CC}_\NIS(G) = \frac{1}{N} \sum_{i=1}^N \left(\frac{1}{Q(u_i)} \times \frac{1}{1+\hat{d}(u_i)} \right), \  u_i \sim Q(u),
\end{equation}
\vspace{-5pt}
where, $\hat{d}(u_i)$ is estimated as:
\begin{equation}\label{eq:Nbar}
\hat{d}(u_i) = \frac{1}{M} \sum_{ j =1}^M \frac{s(u_i, v_{i,j})}{q_{u_i}(v_{i,j})}, \quad  v_{i,1},...,v_{i,M} \sim q_{u_i}(v). 
\end{equation}

The proposal $Q(u)$ biases the distribution towards isolated vertices, i.e., ones that have low degrees, as they contribute the most toward the number of clusters. The second biases the distribution towards sampling nodes that are likely to be connected, i.e., to have high $\hat{s}(u,v)$. Note the only place human feedback is required is in computing $\hat{d}(u)$ in Eq.~\ref{eq:Nbar} consisting of $M$ queries. Hence the overall query complexity of the approach is $N \times M$, similar to the simple Monte Carlo estimator. But if the similarity function $\hat{s}(u,v)$ is a good approximation to the true similarity then we expect the estimator to have a lower variance. Algorithm~\ref{alg:nis} and Fig.~\ref{fig:main} describe the overall scheme.

\vspace{-5pt}
\subsection{Variance and Confidence Intervals}\label{sec:variance} 

\begin{theorem}
\label{theorem:variance}
Assume $Q(u) > 0$ for all $u$ and $q_u(v) > 0$ for all $u, v$. Let $\widehat{\CC}_{N,M}$ denote the estimator using $N$ sampled vertices and $M$ sampled edges per vertex. For any $M > 0$, the estimator $\widehat\CC_{N,M}$ is asymptotically normal, i.e.,
\begin{equation}
\sqrt{N} (\widehat\CC_{N,M} - \mu_{M}) \stackrel{D}{\rightarrow} \mathcal N(0, \sigma^2_{1,M}) \text{ as $N \to \infty$},
\end{equation}
where $\mu_{M} = \E[\widehat{\CC}_{N,M}] = \E[\widehat{\CC}_{1,M}]$ and $\sigma^{2}_{1,M} = \Var(\widehat\CC_{1,M})$. Furthermore, the bias $|\mu_M - \CC|$ is $O(1/M)$. Together these facts imply that the estimator is consistent as both $N$ and $M$ go to infinity, i.e., $\lim_{N \to \infty, M \to \infty} \widehat \CC_{N,M} = \CC$.
\end{theorem}

Asymptotic normality justifies the construction of confidence intervals as follows: let $\hat \sigma^2_{1,M}$ be the sample variance of $\frac{1}{Q(u_i)} \times \frac{1}{1+\hat d(u_i)}$ and use the 95\% confidence interval  $\widehat \CC_{N,M} \pm z_{0.025} \cdot \hat\sigma_{1,M}/\sqrt{N}$.

\begin{proof}[Proof of Theorem~\ref{theorem:variance}]
The estimator $\widehat \CC_{N,M}$ is the sample average of $N$ identically distributed copies of
\vspace{-5pt}
\begin{equation}
\widehat \CC_{1,M} = \frac{1}{Q(u_1)} \times \frac{1}{1 + \hat d(u_1)}.
\end{equation}
Therefore $\E[\widehat{\CC}_{N,M}] = \E[\widehat{\CC}_{1,M}]$ and the asymptotic normality result holds by the central limit theorem as long as $\mu_{M}$ and $\sigma^2_{1,M}$ are finite. The quantity $\frac{1}{Q(u_1)} \times \frac{1}{1 + \hat d(u_1)}$ is bounded above by our assumption that $Q(u) > 0$ and $q_u(v) > 0$. Together with the fact that the support of the joint sample space $(u_1, v_{1,1}, \ldots, v_{1,M})$ is finite, implies $\mu_M$ and $\sigma^2_{1,M}$ are finite. It remains to show that the bias is $O(1/M)$. This can be proved by the delta method, but also follows from a result of \cite{rainforth2018nesting}, which applies to our setting to assert that the mean-squared error of $\widehat{\CC}_{N,M}$ is $O(\frac{1}{N} + \frac{1}{M^2})$. Because mean-squared error equals variance plus squared bias, this necessitates that the bias, which only depends on $M$, is $O(\frac{1}{M})$.
\end{proof}

\vspace{-5pt}
\begin{remark}
\cite{rainforth2018nesting} describe the optimal allocation of samples when $T=NM \to \infty$ and found that it occurs when $N$ grow proportionally to $M^2$ in our setting. However, in applications with finite $T$, the lowest error occurred for $N \propto T^\alpha$ with $\alpha$ between 0.5 and 0.6, which suggests selecting $N$ to be approximately proportional to $M$ may be better in practice (See \S~\ref{sec:parameters}--Parameters for NIS).
\end{remark}

\vspace{-15pt}
\section{Experiments}\label{sec:evaluation}
\vspace{-5pt}
This section describes the experimental setup, datasets, models used for computing image similarity, baselines, and the evaluation metrics.

\vspace{-5pt}
\paragraph{Datasets} 
We evaluate our approach on seven animal re-identification datasets, where the goal is to estimate the number of individuals (see Tab.~\ref{table:datasets}) ---
\textbf{Chimpanzee Faces in the Wild} (CTai and CZoo)~\cite{CTai} CTai contains 5,078 images of 72 individuals living in the Ta\"i National Park in C\^ote d’Ivoire, while CZoo consists of 2,109 recordings of 24 chimpanzees.  
\textbf{IPanda50}~\cite{IPanda50} contains 6,874 images of 50 giant panda individuals,  
\textbf{OpenCows2020}~\cite{OpenCows2020} comprises 4,736 images of 46 Holstein-Friesian cattle individuals,  
\textbf{MacaqueFaces}~\cite{MacaqueFaces} includes 6,280 face images of 34 rhesus macaque individuals, 
\textbf{WhaleSharkID}~\cite{WhaleSharkID} features 7693 images with 543 individual whale shark identifications, 
and \textbf{GiraffeZebraID}~\cite{GiraffeZebraID} with 6925 images of 2056 zebra and giraffe individuals in Kenya.

\vspace{-5pt}
\paragraph{Image Encoding and Similarity} We use image embeddings from the MegaDescriptor~\cite{megadesc}, a Swin-transformer model optimized with ArcFace loss, that beats CLIP~\cite{clip}, DINOv2~\cite{dino}, and ImageNet-1k~\cite{imagenet} pre-trained image encoders on animal Re-ID tasks. Specifically, we consider MegaDescriptor-L-384 for our experiments and present ablation tests with its smaller version MegaDescriptor-B-224 in \S~\ref{sec:results}. 
The normalized cosine similarity between a pair of embeddings $\Phi(x_u)$ and $\Phi(x_v)$ as used as the proposal distribution with $\tau=0.5$ (chosen with CZoo):

\vspace{-3pt}
\begin{equation}
\hat{s}(u,v)=\frac{e^{\hat{c}(u,v)/\tau}}{\sum_ve^{\hat{c}(u,v)/\tau}} \ \ \ \ \text{where} \ \ \ \hat{c}(u,v) = \frac{\Phi(x_u)\cdot\Phi(x_v)}{||\Phi(x_u)||_2||\Phi(x_v)||_2}.
\end{equation}
\vspace{-7pt}

\begin{table}[t]
\caption{\textbf{Population Size Estimation on Animal Re-ID Datasets}. The number of images $|\mathcal{D}|$ and individuals $|\mathcal{Y}|$ per dataset along with the estimated population $k$ for a given amount of human effort (sampled pairs). Estimates using connected components (CoCo), Robust Continuous Clustering~\cite{rcc} (RCC), $k$-means ($k$m), and mean-shift (not shown, but see Fig.~\ref{fig:results_all_datasets}) exhibit high error rates. $k$-means improves with the addition of human feedback (pc$k$m) but the error rates remain high. Farthest-first traversal (FFT) outperforms these methods. The proposed Nested-IS (NIS) surpasses these baselines and yields estimates and confidence intervals that often contain the true value. Error rates (\%) on each dataset (shown for NMC and NIS) are also lower compared to the baseline and Nested Monte Carlo (NMC). Our method demonstrates significant improvement on the challenging WhaleSharkID~\cite{WhaleSharkID} and GirrafeZebraID~\cite{GiraffeZebraID}.}
\label{table:datasets}
\vspace{-5pt}
\centering
\resizebox{\columnwidth}{!}{
\begin{tabular}{lcc|c|c|c|c|c|r|r}
\toprule
 & & &Sampled & \multicolumn{6}{|c}{Estimated $k$ @ $N \times M$ sampled pairs}\\
 & & &pairs & \multicolumn{4}{c}{ } &  \multicolumn{2}{|c}{CC$\pm$CI (\%error)} \\
Dataset & $|\mathcal{D}|$ & $|\mathcal{Y}|$ & ($\frac{N \times M}{|E|}$)& CoCo & RCC & $k$m$\rightarrow$pc$k$m& FFT & NMC & NIS \\
\hline
 CTai~\cite{CTai} & 5078 & 72 & 0.014\% & 7 & 258 & 90 $\rightarrow$ 87 &   54 & 26$\pm$4 (64\%) & \textbf{63}$\pm$\textbf{20 (16\%)}\\
 CZoo~\cite{CTai} & 2109 & 24 & 0.004\% & 3 & 57 & 50 $\rightarrow$ 49 &   13 & 08$\pm$1 (68\%) & \textbf{23}$\pm$\textbf{07 (17\%)} \\
 IPanda50~\cite{IPanda50} & 6874 & 50 & 0.002\% & 7 & 251 & 80 $\rightarrow$ 79 &  34 & 17$\pm$3 (66\%) & \textbf{50}$\pm$\textbf{16 (13\%)} \\
 OpenCows2020~\cite{OpenCows2020} & 4736 & 46 & 0.006\% & 39 & 220 & 70 $\rightarrow$ 69 &  39  & 18$\pm$3 (62\%) & \textbf{40}$\pm$\textbf{13 (17\%)} \\
 MacaqueFaces~\cite{MacaqueFaces} & 6280 & 34 & 0.002\% & 13 & 90 & 80 $\rightarrow$ 77 &  31  & 15$\pm$2 (55\%) & \textbf{34}$\pm$\textbf{08 (09\%)} \\
 WhaleSharkID~\cite{WhaleSharkID} & 7693 & 543 & 0.16\% & 630 & 1182 & 70 $\rightarrow$ 132 &   251  & 145$\pm$9 (73\%) & \textbf{543}$\pm$m\textbf{74 (06\%)}\\ 
 GiraffeZebraID~\cite{GiraffeZebraID} & 6925 & 2056 & 0.16\% & 4714 & 500 & 100 $\rightarrow$ 155 &   271  & 403$\pm$34 (80\%) & \textbf{1951}$\pm$\textbf{370 (08\%)}\\ 
 \hline
 \multicolumn{4}{l|}{Average error (\%)} & 69.4\% &219.2\%  & 80.4\% $\rightarrow$ 75.4\% &  38.2\%  & 66.9\% & \textbf{12.3\%}\\ 
\bottomrule
\end{tabular}
}
\end{table}

\vspace{-5pt}
\paragraph{Clustering Baselines}\label{sec:baselines} 
We employ four baselines for estimating the number of clusters. The first is mean-shift clustering~\cite{meanshift}. We set the bandwidth parameter as the average distance between samples and their nearest neighbor, and estimate the number of clusters as the number of modes. The second method is $k$-means~\cite{kmeans}. To calculate the \emph{optimal} $k$, we calculate the within-cluster sum of squares $\mathcal{J}_{\km} = \sum_{i\in G} ||{c_i}-u_i||^2_2$, where $c_i$ is the cluster center corresponding to $u_i$. The elbow is identified as the value of $k$ where the slope becomes approximately constant. Third, we use connected components (CoCo) after thresholding the similarity values $\hat{s}(u,v)$ of our graph $G$. To find the optimal threshold, we use the same procedure as the $k$-means elbow; that is, we calculate $\mathcal{J}_{\km}$ for different thresholds and return the threshold at the elbow.
Lastly, we use robust continuous clustering (RCC) to extract the number of clusters after optimizing the algorithm's objective using the default hyperparameters in~\cite{rcc}.

\vspace{-5pt}
\paragraph{Active Clustering}\label{sec:activeclustering}
We use active variants of the above approaches. 
First, we use pairwise constrained $k$-means (pc$k$-means~\cite{basu2004}). For a given $k$, the objective is:
\vspace{-10pt}
\begin{multline*}
\mathcal{J}_{\pckm} = \sum_{i \in G} ||c_i-u_i||^2_2 + 
\sum_{(i,j)\in \M} \alpha \mathbbm{1}[l_{i}\neq l_j] +
\sum_{(i,j)\in \C} \beta \mathbbm{1}[l_{i} = l_j],
\end{multline*}

\vspace{-3pt}
where $\M$ and $\C$ are the sets of \emph{must-link} and \emph{cannot-link} constraints, respectively, and $l_i$ denotes the cluster index for $i$. Scalars $\alpha,\beta > 0$ trade off the $k$-means objective with the cost of violating constraints. We select empirically $\alpha=\beta=1$. 
Given a set of constraints, we can pick the optimal $k$ using the elbow heuristic on the $\mathcal{J}_{\pckm}$ objective. We pick samples based on a farthest-first traversal scheme described below.
The second method is constrained mean shift~\cite{cms2022schier} (Constrained-MS) that introduces a density-based integration of the constraints.
Lastly, instead of incorporating constraints within $k$-means one can directly estimate the number of clusters using a farthest-first traversal (FFT) of points and exhaustive comparison with existing individuals. The approach is as follows:
1) Initialize the list of sampled images $S$ as empty and list of discovered individuals as $I$ as empty,
2) Selecting an unsampled point that is farthest from $S$, i.e., $u=\text{argmax}_u\min_{v \in S}~d(u,v)$, and add it to the sampled set $S$.
3) Compare $u$ to all the previously discovered individuals in $I$. If it matches an individual add it to the corresponding list, else start a new list with $u$ and add it to $I$.
FFT rapidly explore the dataset to find at least one member for each cluster. At each iteration it pays a cost equal to the number of discovered individuals.

\begin{figure}[t]
    \centering
    \begin{tabular}{ccccc}
    \vspace{-10pt}
       \multirow{-4}{*}[1pt]{\includegraphics[width=0.030\linewidth]{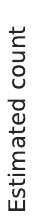}}  &
       \hspace{30pt}\includegraphics[width=0.26\linewidth]{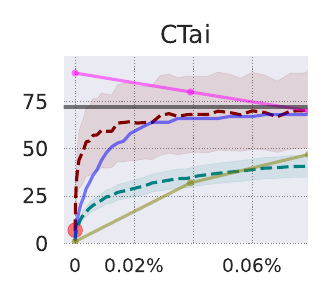}  &  
       \hspace{0pt}\includegraphics[width=0.26\linewidth]{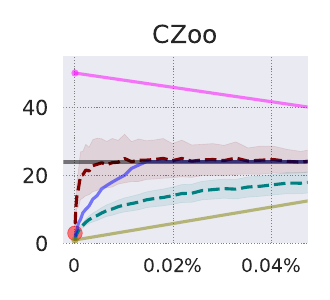} &
       \hspace{0pt}\includegraphics[width=0.26\linewidth]{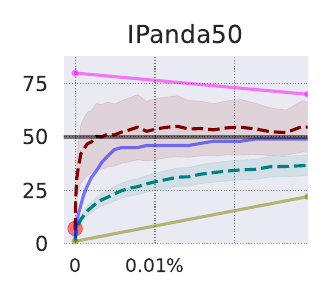} & \\ 
        &
       \hspace{-30pt}\includegraphics[width=0.26\linewidth]{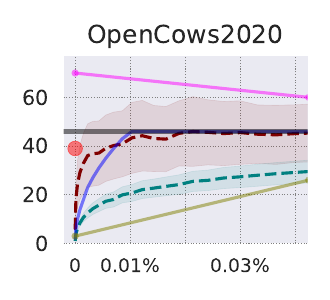}  &  
       \hspace{-80pt}\includegraphics[width=0.26\linewidth]{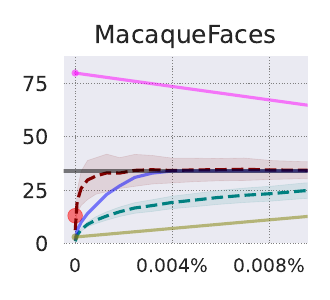} &
       \hspace{-100pt}\includegraphics[width=0.26\linewidth]{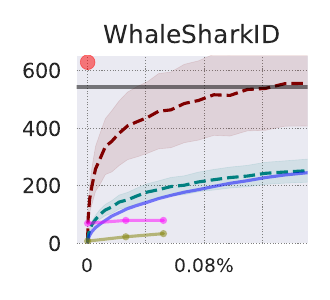} &
       \hspace{-60pt}\includegraphics[width=0.26\linewidth]{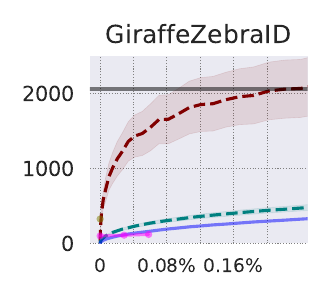} \\
       \multicolumn{5}{c}{\scriptsize \text{\# sampled pairs / \# edges}}\\ 
       \multicolumn{5}{c}{\includegraphics[width=0.6\linewidth]{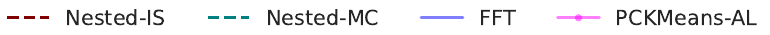}} \\
       \multicolumn{5}{c}{\includegraphics[width=0.4\linewidth]{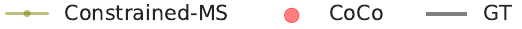}}
    \end{tabular}\\
    \vspace{-10pt}
    \caption{
    \textbf{Performance of Estimating $k$ per Human Effort on Animal Re-ID Datasets.} We use the cosine similarity built from the MegaDescriptor-L-384 image embeddings--See \S~\ref{sec:evaluation}. The human effort is measured as the fraction of the sampled pairs and total pairs $|E|$ in the dataset $G$. 
    Our method estimates the true count with less human effort compared to baselines.
    Dashed lines indicate the mean estimates and shaded regions indicate the mean 95\% confidence interval across 100 trials.
    }
    \vspace{0pt}
    \label{fig:results_all_datasets}
\end{figure}

\vspace{-9pt}

\paragraph{Evaluation Metric and Human Effort}
Error is measured as $|\CC - \widehat{\CC}|/\CC$, the fractional absolute difference between the estimated $\widehat{\CC}$ and the true number of clusters $\CC$. Human effort is measured as the number of pairwise queries used for estimation. To account for variable dataset sizes, we report the number of sampled pairs normalized by the total number of edges $|E|$ in $G$. 

\vspace{-5pt}
\section{Results}\label{sec:results}
\vspace{-5pt}

We compare our method to the baselines described in \S~\ref{sec:baselines} that include connected components (CoCo), pairwise constrained $k$-means (pc$k$-means), constrained mean shift (Constrained-MS), our farthest-first traversal baseline (FFT), and nested Monte Carlo sampling (Nested-MC) from Eq. (\ref{eq:mc}).

\begin{figure}[t]
    \centering
    \begin{tabular}{ccc}
      \hspace{-8pt}\includegraphics[width=0.35\linewidth]{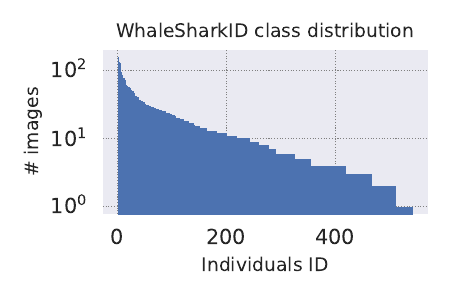}  &   
      \hspace{-10pt}\includegraphics[width=0.35\linewidth]{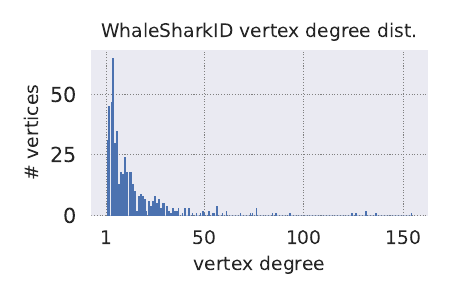} &
      \multirow{-7.5}{*}[1pt]{\hspace{0pt}\includegraphics[width=0.33\linewidth]{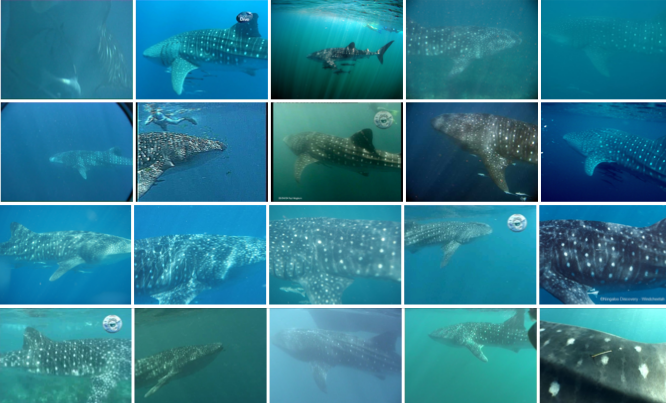}}\\
    \end{tabular}\\
    \vspace{-10pt}
    \caption{\textbf{WhaleSharkID Dataset Statistics.} (Left) The dataset is long-tailed with many individuals with a few images. (Center) Histogram of vertex degree distribution--most individuals have less than 5 images. (Right) Sample images from the dataset. 
    }
    \vspace{-10pt}
    \label{fig:whaleshark}
\end{figure}

\vspace{-7pt}
\paragraph{Baselines Perform Poorly} 
Techniques like $k$-means and CoCo perform poorly, with about 80\% and 69\% error on average, respectively, despite using state-of-the-art image embeddings to compute similarity (see Fig.~\ref{fig:results_all_datasets} and Tab.~\ref{table:datasets}). For instance, $k$-means estimates $70$ clusters for WhaleSharkID instead of $543$, $70$ clusters instead of $34$ for OpenCows2020, and $50$ clusters instead of $24$ for CZoo. Estimates using CoCo are slightly better on average but tend to be an underestimate. One reason is that choosing hyperparameters for clustering is particularly challenging. For instance, the WCSS curve as a function of $k$ looks rather smooth, and there is no clear ``elbow.'' Another reason is that the underlying similarity is imperfect. 
Improving this through deep metric learning approaches would require significant resources and expertise to set various parameters and may not pan out when supervision is limited~\cite{musgrave2020}.

\vspace{-7pt}
\paragraph{Nested-IS outperforms Nested-MC} 
Both Nested-MC and Nested-IS improve over the baselines as more human feedback is collected. As described in \S~\ref{sec:nestedsampling}, Nested-MC samples edges uniformly at random, while Nested-IS is driven by the pairwise similarity and leads to an error reduction at a much faster rate as seen in Fig.~\ref{fig:results_all_datasets}.  
Specifically, we get a relative error reduction of 82\%  
over Nested-MC with the same human effort. 

\begin{figure}[t]
    \centering
    \begin{tabular}{cc}
       \includegraphics[width=0.403\linewidth]{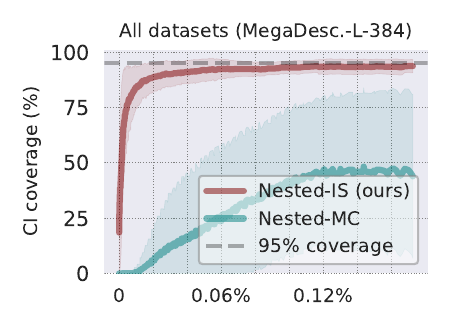}  &    
       \hspace{-14pt}\includegraphics[width=0.40\linewidth]{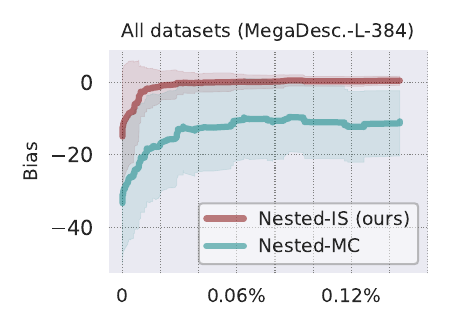}\\
    \end{tabular}\\
    \vspace{-5pt}
    \scriptsize \text{\# sampled pairs / \# edges}\\
    \vspace{-5pt}
    \caption{\textbf{Confidence Intervals' Coverage and Bias} using MegaDescriptor-L-384 feature embeddings on all datasets. 
    (Left) Our method produces a CI coverage close to $95\%$ (when using $z_{0.025} = 1.96$ for a $95\%$ confidence interval--\S~\ref{sec:variance}) with less than $0.04\%$ of sampled pairs.
    (Right) Even though our estimator has a negative bias for a low number of sampled pairs, it rapidly reaches zero compared to Nested-MC.
    }
    \vspace{0pt}
    \label{fig:ci_coverage_bias}
\end{figure}

\vspace{-7pt}
\paragraph{Nested-IS outperforms Active Clustering} 
We compare our method to pairwise constrained $k$-means (pc$k$-means-AL), constrained mean shift (Constrained-MS), and the proposed farthest-first traversal approach (FFT) with similar human effort. Fig.~\ref{fig:results_all_datasets} shows that 
Nested-IS handily outperforms active clustering. 
In WhaleSharkID, despite performing best, our method still requires querying relatively many pairs to get accurate estimates, which can be explained by the difficulty of the task.
In Fig.~\ref{fig:whaleshark}--\emph{left-center} we show the histogram of individuals with a clear long-tailed distribution and the histogram of vertex degrees where we can see that in WhaleSharkID most individuals have less than 5 images. The performance of the MegaDescriptor is also the lowest (around 62\%) on the Re-ID task on this dataset. However, the performance savings over alternative approaches is significant.

\begin{figure}[t]
    \centering
    \begin{tabular}{ccc}
      \includegraphics[width=0.35\linewidth]{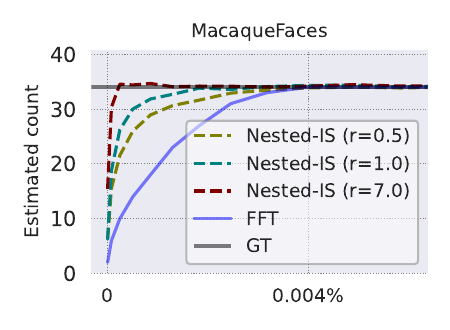}  &    
      \includegraphics[width=0.28\linewidth]{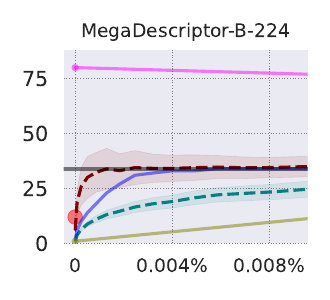} &
      \hspace{-10pt}\includegraphics[width=0.28\linewidth]{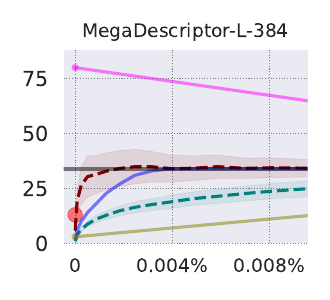}\\
      \multicolumn{1}{c}{\scriptsize \text{\# sampled pairs / \# edges}} &
       \multicolumn{2}{c}{\includegraphics[width=0.55\linewidth]{figures/eccv_fig_legend_top.pdf}} \\
      & \multicolumn{2}{c}{\includegraphics[width=0.4\linewidth]{figures/eccv_fig_legend_bottom_.pdf}}
    \end{tabular}\\

    \vspace{-10pt}
    \caption{ 
    \textbf{Ablation Experiments on MacaqueFaces} (Left) Our proposed method (Nested-IS) with different sampling ratios $r=M/N$. Increasing the number of sampled neighbors $M$ per sampled vertices $N$ (\S~\ref{sec:nestedis}) improves performance for these datasets. (Center-Right) Performance comparison between MegaDescriptor-B-224 and MegaDescriptor-L384. Although our method is slightly better using superior feature embeddings (L-384), it still outperforms all baselines with B-224. 
    }
    \vspace{0pt}
    \label{fig:ablation}
\end{figure}

\vspace{-7pt}
\paragraph{Confidence Intervals are Calibrated}  
In addition to lower error rates, a key advantage of our approach is that it also provides confidence intervals (CIs) (See Fig.~\ref{fig:results_all_datasets}).
To estimate the quality of the CI estimation we calculate the empirical coverage of the CI over 100 runs (i.e., the percentage of times the estimated CIs contains the true count). Our method produces close to $95\%$ coverage using $z_{0.025} = 1.96$ for a $95\%$ CI, as described in \S~\ref{sec:variance}, when around $0.02\%$ of pairs are sampled across datasets, as shown in Fig.~\ref{fig:ci_coverage_bias}--\emph{left}.

\vspace{-9pt}
\paragraph{Estimation Bias is Low} We calculate the empirical bias, denoted as $\textsc{Bias} = \E[\widehat{\CC}_\NIS(G)] - \CC(G)$ where $\E[\widehat{\CC}_\NIS(G)]$ is the mean count estimate over 100 runs. Fig.~\ref{fig:ci_coverage_bias}--\emph{right} shows that the bias is negative initially, but it rapidly drops to 0. The initial negative bias might explain the lower coverage of the CIs.

\vspace{-7pt}
\paragraph{Parameters for Nested-IS}\label{sec:parameters} 
There is a trade-off between the number of sampled vertices $N$ and sampled neighbors per vertex $M$. 
Increasing $N$ will reduce the variance of the estimation of the number of clusters, and increasing $M$ will reduce the variance of each vertex degree estimate. In Fig.~\ref{fig:ablation}--\emph{left} we show that increasing the the ratio between the number of sampled neighbors per vertex $M$ relative to the number of sampled vertices $N$ (i.e., $r=M/N$) produces more accurate estimations with less human effort. For instance, using $r=7$ achieves close-to-zero error with $1/5$ of the sampled pairs compared to using $r=0.5$.

\vspace{-7pt}
\paragraph{Ablation with other MegaDescriptors} We test the performance of all baselines using a smaller less-performing version of the MegaDescriptor (MegaDescriptor-B-224). In Fig.~\ref{fig:ablation}--\emph{center-right} we show that our method performs similarly with both MegaDescriptor versions.

\vspace{-7pt}
\paragraph{How Much do we Know about the Clustering?}
Sometimes we also wish to recover the clustering in a dataset, such as in generalized category discovery~\cite{gcd}. Should human effort be used to estimate the clustering directly using active clustering, or to estimate $k$ for $k$-means using our approach? To answer this we conduct the following experiment. First, we can measure the accuracy of a clustering as:
\vspace{-5pt}
\begin{equation}
    \textsc{Acc} = \max_{p\in\mathcal{P}(\mathcal{Y})} \frac{1}{N} \sum_{i=1}^N \mathbbm{1}[y_i = p(\hat{y_i})],
\end{equation}
where $\mathcal{P}(\mathcal{Y})$ is the set of all permutations of the class labels. We compare the accuracy of clustering using $k$-means with estimated $k$ using the elbow method, using pc$k$-means, and using $k$-means with $k$ estimated using Nested-IS for tthe same amount of human effort as pc$k$-means. 
Fig.~\ref{fig:kmeans_accuracy} shows the results. 
Surprisingly, we find that accurately estimating $k$ has a larger impact on the quality of clustering than active clustering, which suggests that human effort is better spent to estimate $k$ initially. In the Supplementary Material we describe experiments on GCD on fine-grained classification datasets and show a similar trend.

\begin{figure}[t]
    \centering
    \includegraphics[width=1.0\linewidth]{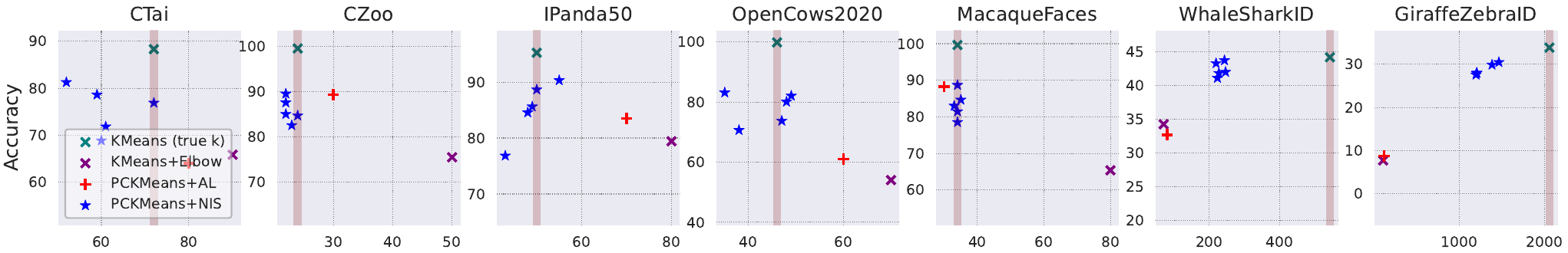}\\
    \vspace{0pt}
    \scriptsize{\text{\# clusters (k)}}\\
    \vspace{-5pt}
    \caption{ 
    \textbf{Measuring Clustering Accuracy}. Pairwise constraints (pc$k$-means+AL) improves the clustering accuracy over $k$-means+elbow, but the clustering accuracy with our estimated $k$ is even better (pc$k$-means+NIS).
    We plot results with five independent runs for our proposed approach (blue stars). 
    The red shaded line indicates the true number of clusters in the dataset.
    }
    \vspace{0pt}
    \label{fig:kmeans_accuracy}
\end{figure}

\vspace{-5pt}
\section{Discussion and Conclusion}\label{sec:discussion}
\vspace{-5pt}
We propose a human-in-the-loop approach to estimate population size when deploying imperfect Re-ID systems. 
By carefully selecting a small fraction of pairs to label (often less than 0.002\% of all edges), our approach produces unbiased estimates of the population size. A key advantage of our method is that it generates confidence intervals which can be used for guiding human effort. This approach can be implemented on top of any Re-ID system, as it requires only a pairwise similarity between images, making it practical for low-resource settings. Our approach adds to the growing literature on statistical estimation techniques~\cite{iscount,discount,ppi} that combine model predictions and ground-truth labels to improve the precision of count estimates. However, we tackle the novel problem of estimating cluster counts, which involves pairwise comparisons.

\section*{Acknowledgments}
This material is based upon work supported by the National Science Foundation under Grants  \#1749854, \#1749833, and \#2329927. We thank Hung Le and Cameron Musco for initial discussions and the Wildlife Datasets team for publicly releasing the datasets and models for animal Re-ID.

%
%

\bibliographystyle{splncs04}
\bibliography{main}

\clearpage
\appendix
\counterwithin{figure}{section}
\counterwithin{table}{section}
\renewcommand\thefigure{\thesection\arabic{figure}}
\renewcommand\thetable{\thesection\arabic{table}}
   
\setcounter{page}{1}
\setcounter{figure}{0}
\setcounter{table}{0}
\setcounter{section}{0}

\section{Experiments on Fine-Grained Classification Datasets} \label{sec:fgc}
\vspace{-10pt}

In this section we present experiments on four fine-grained image classification datasets, where the goal is to estimate the number of categories ---
\textbf{Caltech-UCSD Birds} (CUB)~\cite{cub} consists of 11,788 images of 200 bird species, 
\textbf{Stanford Cars}~\cite{cars} contains 16,185 images of 196 car models, 
\textbf{FGVC Aircraft}~\cite{aircraft} comprises 10,000 images of 100 aircraft models,  
and \textbf{Oxford Flowers}~\cite{flowers} includes 8,189 images of 102 flower categories. Although the names of categories are known, we use the same pairwise setting as the Re-ID tasks.

\vspace{-10pt}
\subsection{Performance of Estimating $k$}
\vspace{-5pt}

Here we show the estimated $k$ as a function of the number of sampled pairs, as with the Re-ID datasets in Fig.~\ref{fig:results_all_datasets}. Similarly to Re-ID tasks, our method outperforms all the baselines when using feature embeddings from an ImageNet-pretrained ResNet50, DINO ViT-B/8, and CLIP ViT-L/14. We calculate the similarity as described in \S~\ref{sec:evaluation}.

\begin{figure}[h]
\vspace{-10pt}

    \centering
    
    \begin{tabular}{ccccc}
    \hspace{-8pt}
    \vspace{-10pt}
       \multirow{-1}{*}[1pt]{\includegraphics[width=0.035\linewidth]{figures/eccv_fig_ylabel.pdf}}  &
       \hspace{-5pt}\includegraphics[width=0.27\linewidth]{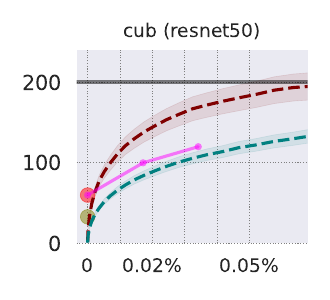}  &  
       \hspace{-12pt}\includegraphics[width=0.27\linewidth]{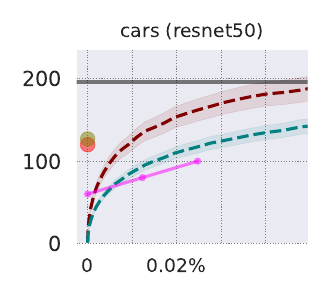}  &
       \hspace{-12pt}\includegraphics[width=0.27\linewidth]{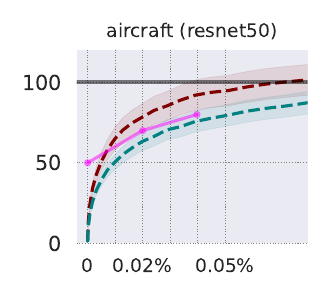} &
       \hspace{-12pt}\includegraphics[width=0.27\linewidth]{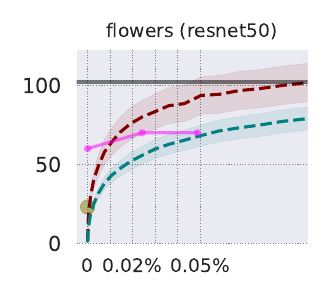} \\ 
        \vspace{-10pt}
        &
       \hspace{-5pt}\includegraphics[width=0.27\linewidth]{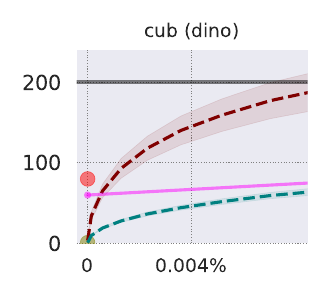}  &  
       \hspace{-12pt}\includegraphics[width=0.27\linewidth]{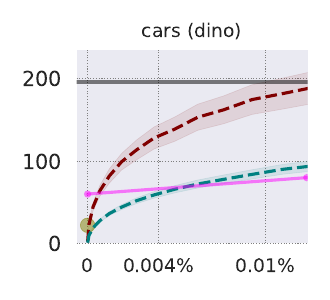}  &
       \hspace{-12pt}\includegraphics[width=0.27\linewidth]{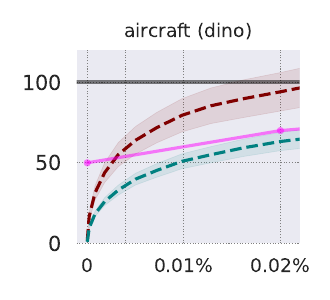} &
       \hspace{-12pt}\includegraphics[width=0.27\linewidth]{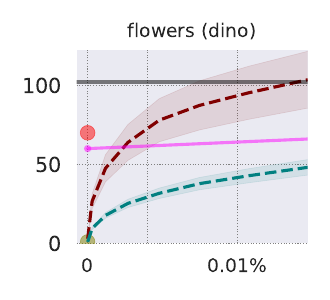} \\
       \vspace{-5pt}
        &
       \hspace{-5pt}\includegraphics[width=0.27\linewidth]{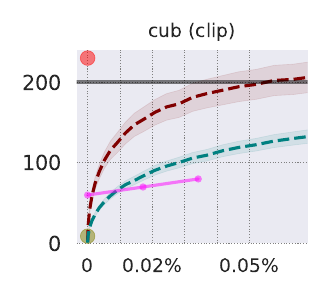}  &  
       \hspace{-12pt}\includegraphics[width=0.27\linewidth]{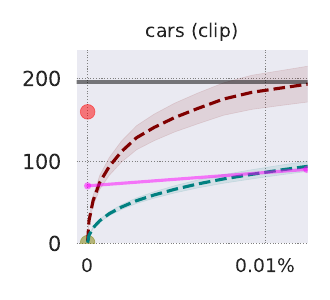}  &
       \hspace{-12pt}\includegraphics[width=0.27\linewidth]{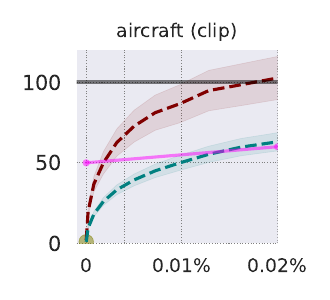} &
       \hspace{-12pt}\includegraphics[width=0.27\linewidth]{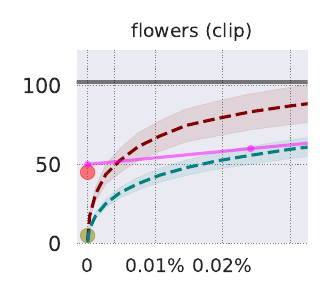} \\
       \multicolumn{5}{c}{\scriptsize \text{\# sampled pairs / \# edges}}\\ 
       \multicolumn{5}{c}{\includegraphics[width=0.6\linewidth]{figures/eccv_fig_legend_top.pdf}} \\
       \multicolumn{5}{c}{\includegraphics[width=0.4\linewidth]{figures/eccv_fig_legend_bottom_.pdf}} 
    \end{tabular}\\
    \vspace{-10pt}
    \caption{
    \textbf{Performance of Estimating $k$ per Human Effort} across fine-grained classification datasets. We use the cosine similarity built from ImageNet pretrained ResNet50, DINO ViT-B/8, and CLIP ViT-L/14 image embeddings. The human effort is measured as the fraction of the sampled pairs and total pairs $|E|$ in the dataset $G$. 
    Our method estimates the true count with less human effort compared to the other baselines.
    Dashed lines indicate the mean estimates and shaded regions indicate the mean 95\% confidence interval across 100 trials.
    }
    \vspace{0pt}
    \label{fig:results_all_class_datasets}
\end{figure}

\subsection{Measuring Clustering Accuracy}
\vspace{-5pt}

Similarly to the Re-ID datasets (See Fig.~\ref{fig:kmeans_accuracy}), we find that accurately estimating $k$ has a bigger impact on the quality of clustering than active clustering, which suggests that human effort is better spent to estimate $k$ initially.

\begin{figure}[h]
    \centering
    \includegraphics[width=0.55\linewidth]{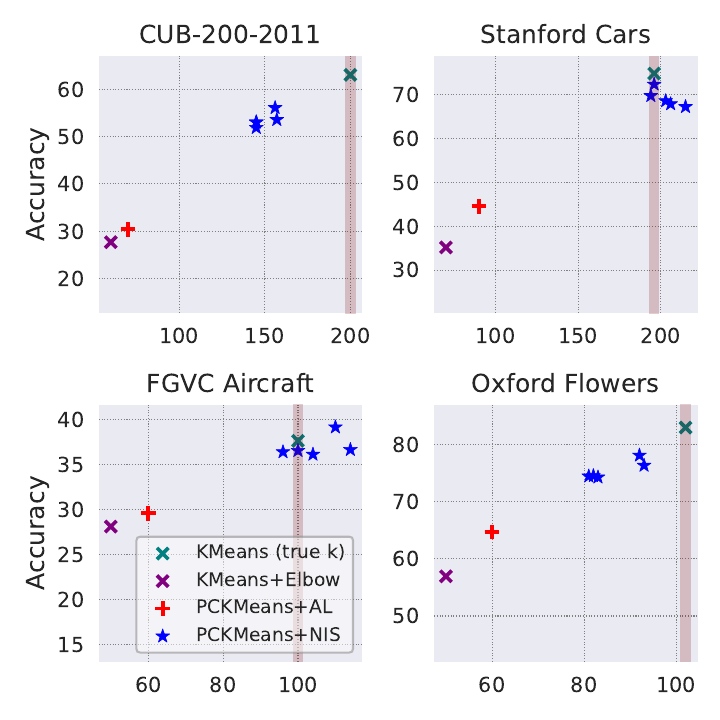}\\
    \footnotesize{\text{\# clusters (k)}}\\
    \vspace{-5pt}
    \caption{\textbf{Clustering accuracy on fine-grained classification datasets} using the right number of clusters. Using the improved $k$ and pairwise constraints (pc$k$-means+AL) improves the clustering accuracy over $k$-means+elbow, while clustering accuracy with our estimated $k$ improves accuracy further (pc$k$-means+NIS).
    We plot results with five estimated $k$s and constraints from our proposed approach (blue stars). 
    The red shaded line indicates the true number of clusters in the dataset.
    }
    \vspace{0pt}
    \label{fig:kmeans_accuracy_sup}
\end{figure}

\end{document}